\documentclass[runningheads]{llncs}

\input{main-includes.tex}
\usepackage{bibtopic}

\usepackage{graphicx}
\usepackage{microtype}

\begin{document}
\title{String Theories involving Regular Membership Predicates: From Practice to Theory and Back}
\titlerunning{Decidability involving Regular Membership Predicates}
\author{Murphy~Berzish\inst{1} \and
		Joel~D.~Day\inst{2} \and
		Vijay~Ganesh\inst{1} \and
	    Mitja~Kulczynski\inst{3} \and
	    Florin~Manea\inst{4} \and
	    Federico~Mora\inst{5} \and
	    Dirk~Nowotka\inst{3}
	}
\authorrunning{M. Berzish et al.}
\institute{University of Waterloo, Waterloo, Canada\and
  		   Loughborough University, Loughborough, UK\and
	       Kiel University, Kiel, Germany \and
  		   University of G\"ottingen and Campus-Institute Data Science, G\"ottingen, Germany \and
		   University of California, Berkeley, USA
}
\maketitle

\begin{abstract}
Widespread use of string solvers in formal analysis of string-heavy programs has led to a growing demand for more efficient and reliable techniques which can be applied in this context, especially for real-world cases. Designing an algorithm for the (generally undecidable) satisfiability problem for systems of string constraints requires a thorough understanding of the structure of constraints present in the targeted cases. In this paper, we investigate benchmarks presented in the literature containing regular expression membership predicates, extract different first order logic theories, and prove their decidability, resp. undecidability. Notably, the most common theories in real-world benchmarks are $\mathsf{PSPACE}$-complete and directly lead to the implementation of a more efficient algorithm to solving string constraints. 
\end{abstract}

\section{Introduction}
String constraint solving (for short, \emph{string
solving}) is a topic within the more general constraint solving area, where one
is interested in checking the satisfiability of particular quantifier-free first
order logic formulae over  a structure involving string equalities, linear
arithmetic over string length, and regular language membership, all built on top
of string variables. While deeply rooted in algebra and combinatorics on words
(more precisely, in the theory of word equations \cite{diekert2015more}), in
recent years, string solving has also attained widespread interest in the formal
methods community. Indeed, this model arises naturally in, e.g., tasks related
to formal analysis of string-heavy programs such as sanitization and validation
of inputs (cf. \cite{kaluza}), 
leading to the development of
multiple string solvers such as \textsc{CVC4}~\cite{barrett2011}, \textsc{Z3seq}~\cite{bjorner2012smt},
\textsc{Z3str3}~\cite{berzish2017z3str3},
and \textsc{Woorpje}~\cite{woorpjesat}. Even though these solvers are quite
efficient for certain practical use cases, novel applications demand
even more \emph{efficient} and \emph{reliable techniques}, especially for
real-world inputs. Taking a closer look at all reported security-related
vulnerabilities listed in the Common Vulnerabilities and Exposures Repositories~\cite{CVEDB}, the most frequent issues are related to strings, e.g. 
\emph{Cross-site Scripting}. For such an attack, an attacker
inserts malicious data into an HTML document, which is usually countered by input
sanitization using regular expressions. However, due to the complex
requirements, coming up with correct regular expressions is error-prone. Consider
for example the regular expression \texttt{/[\^\,\!A-Za-z0-9 .-@:/]/} taken from
\cite{bultan2017string} which was used inside the PHP web application
MyEasyMarket~\cite{balzarotti2008saner} to sanitize a user's input. The
intention of the developer was to remove everything other than alphanumeric
characters and the symbols \texttt{.}, \texttt{-}, \texttt{@}, \texttt{:}, and
\texttt{/}. Unfortunately, this expression overlooks the special semantics of \texttt{-} within
a regular expression. Instead of listing all unwanted symbols individually (\texttt{.-@}),
this regex specifies the union of all characters between \texttt{.} and \texttt{@}. Since \texttt{<} is within this range, an attacker can inject HTML elements which bypass the sanitizer. The correct
regular expression using proper escaping has the form
\texttt{/[\^\,\!A-Za-z0-9 .\textbackslash-@:/]/}. Detecting these kinds of errors
is extremely hard. This is where modern string solvers come into play.
Based on a proper specification, a string solver that handles regular
expression membership predicates is able to reveal human mistakes as seen above.  

Theories containing string constraints have
been studied for decades. In \cite{makanin1977} Makanin proved that the
satisfiability of word equations is decidable. Recently, Je{\.z} \cite{jez2017}
showed that word equations can be solved in non-deterministic linear space.
In~\cite{matiyasevich1968} a
reduction from the more powerful theory of word equations with linear length
constraints (i.e., linear relations between word lengths) to Diophantine
equations is shown. Whether this extended theory of word equations is decidable
remains a major open problem. Solely considering the theory of regular
expression membership predicates, an elegant proof of their is decidability is
given in \cite{norn}. The theory of word equations and regular expression membership 
predicates is known to be decidable~\cite{Lothaire}. 
It is not known if the satisfiability problem for string constraints
involving all aforementioned theories is decidable or not. However, already in
the presence of other simple and natural constraints, like string-number conversion, this problem becomes undecidable
(cf. \cite{rp2018strings}). \looseness=-1

Driven by practical relevance and the need of more efficient algorithms, we analysed
\totalinstances{} string solving instances from industrial applications and solver developers
containing regular expression membership predicates, gathered in
\cite{zaligvinder}, and identified numerous relevant sub-theories based around
regular membership predicates. In particular, we identified theories which may
have a string-number conversion predicate $\numstr{}$ (contains pairs of
integers and their string representation), a string length function and/or string concatenation,
and prove decidability resp.
undecidability for certain sub-theories. The value of this theoretical analysis
of present data is massive, since the sub-theory occurring the most within the
benchmarks is actually $\mathsf{PSPACE}$-complete, as we show within this work.
Most notably, these results lead to an algorithm implemented within
\textsc{Z3str3} showing superior performance compared to its
competitors~\cite{cavpaper}. The algorithm itself was directly informed by the
ideas we used in proofs of the theorems presented in this work. 
Within this paper we show that the theory of complement-free-regular expression membership predicates, with linear length constraints and concatenation is
$\mathsf{PSPACE}$-complete. Furthermore, if we additionally allow
complement, we prove decidability and a $\mathsf{NSPACE}(f(n))$ lower bound, where $f(n)$ is a tetration $2\uparrow^h (cn)$  whose height $h$ depends on the number of stacked complements (and $c$ is a constant).
Continuing this trail, we prove $\mathsf{PSPACE}$-completeness for the theory of complement-free regular
expression membership predicates and a string-number
conversion predicate, which naturally leads to decidability when considering
complements. We show corresponding lower bounds in this case too. At the opposite end of our spectrum, we show that the theory of regular expression membership predicates, linear length constraints, concatenation and string-number conversion is in fact undecidable.

To summarize, our analysis of the benchmarks not only revealed these
theories, but also shows that most considered real-world string constraints
actually fall into a decidable fragment. Out of \totalinstances{}, about 51\% lay
in a decidable fragment. Only considering string constraints without word
equations (30540 of \totalinstances{} instances), 26140 of these instances
(85\%) fall into a decidable fragment. Therefore, our theoretical analysis gives
an intuition wrt. the performance of our solver. \looseness=-1

\section{Preliminaries}
\label{sec:prelim}
Let $\mathds{N}$ be the set of natural numbers (including $0$). By $\domain{}(r)$ we denote the domain of a function $r$. 
An \emph{alphabet} $\symbols$ is a set of symbols, whereas $a \in \symbols$ are called \emph{letters}. By $\symbols^*$ 
we denote the set of all finite words over $\symbols$ and let $\epsilon \in \symbols^*$ denote the \emph{empty word}. For $n \in \mathds{N}$ let $w = a_1\dots a_n \in A^*$ be a word, i.e. a finite sequence. 
By $w[i] = a_i$ we refer to the letter in the $i^{\text{th}}$ position of $w$. 
Let $|w| = n$ denote the \emph{length} of a word $w$.
Let $\symbols'$ be an alphabet. A mapping $h : \symbols^* \rightarrow \symbols'^*$ satisfying $h(uv) = h(u)h(v)$ for all $u,v \in \symbols^*$ is called a \emph{morphism}. In particular, for a morphism $h$ we have $h(\epsilon) = \epsilon$ and by defining $h$ for each $a \in \symbols$ the mapping is completely specified. 

A finite automaton is a structure $A = (Q,\symbols{},\delta,q_0,F)$ where $Q$ is the set of states, $\symbols{}$ an alphabet, $\delta : Q \times \symbols{} \rightarrow 2^{Q}$ a transition function, $q_0$ an initial state, and $F\subseteq Q$ a set of accepting states. We call $A$ a \emph{deterministic finite automaton} (DFA) if for all $q\in Q$ and $a\in\symbols{}$ we have $(q,a)\in\domain{}(\delta)$ and $|\delta(q,a)| = 1$. Otherwise, $A$ is a \emph{non-deterministic finite automaton} (NFA). We say $A$ \textit{accepts} a word $w \in \symbols{}$ if there is a path via $\delta$ leading from $q_0$ to some $f \in F$ (shortly $w \in L(A)$).
We define regular expressions $\regl{\symbols}{}$ over three operations, namely concatenation $\cdot : \regl{\symbols}{\variables} \times \regl{\symbols}{\variables} \rightarrow \regl{\symbols}{\variables}$, union $\cup : \regl{\symbols}{\variables} \times \regl{\symbols}{\variables} \rightarrow \regl{\symbols}{\variables}$, and Kleene star $\!\!\phantom{a}^\ast : \regl{\symbols}{\variables} \rightarrow \regl{\symbols}{\variables}$. On top of these operations we define the set of regular expressions $\regl{\symbols}{\variables}$ inductively as follows: we have $\epsilon,\emptyset,a \in \regl{\symbols}{\variables}$ for $a \in \symbols$.
Given $R_1,R_2 \in \regl{\symbols}{\variables}$ we have $R_1 \cdot R_2, R_1\cup R_2, R_1^* \in \regl{\symbols}{\variables}$. The semantics $L : \regl{\symbols}{\variables} \rightarrow 2^\pats{\symbols}{\variables}$ are given by $L(a) = \Set{a}$ for $a \in \symbols\cup\Set{\epsilon}$
, $L(\emptyset) = \emptyset$. For $R_1,R_2 \in \regl{\symbols}{\variables}$, let  $R_1 \cdot R_2 = \Set{\alpha \cdot \beta | \alpha \in L(R_1), \beta \in L(R_2)}$, $R_1\cup R_2 = L(R_1) \cup L(R_2)$, and $L(R_1^*) = L(R_1)^*$.

We shall generally distinguish between two alphabets, namely a finite set $\terminals = \Set{a,b,c,\dots}$ called \emph{terminals} or \emph{constants} and a possibly infinite set $\variables = \Set{\var{x_1},\var{x_2},\dots}$ called \emph{variables} such that $\terminals \cap \variables = \emptyset$. We call a word $\alpha \in (\variables{} \cup \terminals{})^*$ a \emph{pattern}. Let $\pats{\terminals}{\variables} = (\variables{} \cup \terminals{})^*$ denote the set of all patterns and $\vars{\alpha} \subseteq \variables{}$ denotes the set of all variable occurring in $\alpha$.

We consider first-order logical theories of the $\Sigma_1$ fragment and, if not mentioned, stick to the notation of \cite{ebbinghaus2005finite}. Keep in mind, that whenever the connection of constants $c^\terminals$, functions $f^\terminals$, or relations $R^\terminals$ to a $\mathcal{V}$-structure is clear from context we omit the superscript $\!\!\!\phantom{a}^\terminals$ and simply write $c$, $f$, and $R$, instead of $c^\terminals$, $f^\terminals$, and $R^\terminals$, respectively. Let $\folstruc$ be a $\mathcal{V}$-structure having the domain $\terminals$. An \emph{assignment} $\substitution : \terminals \cup \variables \rightarrow \terminals$ is a morphism such that $\substitution(\var x) \in \terminals^*$ and $\substitution(\var c) = c^\terminals$ holds. The morphism naturally extends to $\pats{\terminals{}}{arg}$.
Let $\assignments{\folstruc} = \Set{\substitution | \substitution: \pats{\terminals{}}{arg} \rightarrow \terminals \text{ morphism}, \forall\;c\in\terminals:\substitution(c) = c^\terminals}$ denote the set of all assignments.
We call a $\mathcal{V}$ formula $\varphi$ in a $\mathcal{V}$-structure $\folstruc$ \emph{satisfiable} if there exists an assignment $\substitution \in \assignments{\folstruc}$ such that $\folstruc, \substitution \models \varphi$ holds and use $\folstruc \models \varphi$ as a short form. In this case we also call $\substitution$ a \emph{solution} to $\varphi$. Consequently, we call $\varphi$ \emph{unsatisfiable} if there does not exist an assignment $\substitution \in \assignments{\folstruc}$ such that $\folstruc, \substitution \models \varphi$ holds and shortly write $\folstruc \not\models \varphi$. A set $\Phi \subseteq \mathsf{FO}(\mathcal{V})$ of $\mathcal{V}$ formulae is satisfiable within a $\mathcal{V}$-structure $\folstruc$ if there exists an assignment $\substitution \in \assignments{\folstruc}$ such that $\folstruc,\substitution \models \varphi$ holds for all $\varphi \in \Phi$ and we denote this by $\folstruc \models \Phi$. Otherwise, the set of formulae $\Phi$ is unsatisfiable within the $\mathcal{V}$-structure $\folstruc$ ($\folstruc \not\models \Phi$). 
As commonly known, the $\Sigma_1$ fragment is as expressive as the quantifier-free fragment of the corresponding theory, and we refer to the quantifier-free fragment whenever we are talking about a specific assignment. \looseness=-1

The theory of \emph{word equations} is built on top of the vocabulary $\mathcal{W} = \Set{\cdot/\!\!/2, \dot{\epsilon}}$ having the axioms of $(\pats{\terminals{}}{},\cdot^{\terminals{}},\epsilon)$ forming a monoid. We consider the $\mathcal{W}$-structure $\weqstructure = \Set{\terminals^*,\cdot^\terminals,\dot{\epsilon}^\terminals}$, whereas $\cdot^\terminals$ is defined as the concatenation of words. For  $\mathcal{W}$ terms $\alpha,\beta \in \pats{\terminals}{\variables}$ the only atom $\alpha \doteq \beta$ is called a \emph{word equation}. Let $\substitution{} \in \assignments{}$. The semantics of a word equation $\alpha \doteq \beta$ are induced through $\substitution{}$ by $\substitution{}(\alpha) = \substitution{}(\beta)$, meaning $\substitution{}$ unifies both sides of the word equation.

The basis theory involving a regular expression membership predicate called \emph{simple regular expressions} is defined on top of the vocabulary 
$\mathcal{R}_s = \{\,\cdot/\!\!/2,\cup/\!\!/2,$ $\!\!\phantom{a}^\ast\!/\!\!/1,\rein/2,\dot{\emptyset}, \dot{\epsilon}\,\}$ being axiomatized as 
\begin{enumerate*}
	\item the existence and associativity of a neutral element $\dot{\epsilon}$ of $\cdot/\!\!/2$,
	\item the existence, associativity, and commutativity of a neutral element $\dot{\emptyset}$ and idempotents,
	\item the distributivity,	\item the annihilation by $\dot{\emptyset}$,
\end{enumerate*}
We consider the many-sorted $\mathcal{R}_s$-structure
$\mathcal{A}_{s} = \{\,\regl{\terminals}{\variables},\terminals^*,$ $
									\cdot^{\terminals{}},
\cdot^{\regl{\terminals}{\variables}},\,$ $
\cup^{\regl{\terminals}{\variables}},$ $
\!\!\phantom{a}^{\ast^{\regl{\terminals}{\variables}}},
\dot{\emptyset}^{\regl{\terminals}{\variables}}, 
\dot{\epsilon}^{\regl{\terminals}{\variables}},
\rein^{\terminals\, \regl{\terminals}{\variables}}
\,\}.$
Our regular expression operations and constants over $\regl{\terminals}{\variables}$ are defined as given before. The semantics of our relation $\rein^{\terminals\, \regl{\terminals}{\variables}}$ is defined by $\alpha \rein^{\terminals\, \regl{\terminals}{\variables}} R$ iff there exists a solution $\substitution{}\in\assignments{\terminals}$ s.t. $\substitution{}(\alpha) \in L(R)$ for $\alpha \in \terminals^* \cup \variables{}$ and $R \in \regl{\terminals{}}{}$. Both theories can be combined by considering the union of their components and denote the structure by $\mathcal{A}^{\doteq}_s$.

\section{From Practice to Theory}
\label{sec:identification}
During the development of an extension to cope with regular membership
constraints within our SMT solver \toolname{}~\cite{cavpaper} we analysed a huge
set of over 100,000 industrial influenced benchmarks gathered by the authors of
\text{ZaligVinder}~\cite{zaligvinder} and identified 22425 instances containing
at least one regular expression membership constraint. This set includes
instances from the AppScan~\cite{Z3str2-FMSD}, BanditFuzz,\footnote{The
BanditFuzz benchmark was obtained via private communication
with the authors.} JOACO~\cite{thome2018integrated}, Kaluza~\cite{kaluza},
Norn~\cite{norn}, Sloth~\cite{sloth}, Stranger~\cite{yu2010}, and
Z3str3-regression~\cite{berzish2017z3str3} benchmarks. Additionally we generated
19979 benchmarks based on a collection of real-world regex queries collected by
Loris D'Antoni from the University of Wisconsin, Madison, USA. Thirdly, we
applied StringFuzz's~\cite{stringfuzz} transformers to instances supplied by
Amazon Web Services related to security policy validation to obtain roughly
15000 instances.

We analysed the benchmarks according to their structure, as well as predicates and functions. 
We identified sets which contain
string-number conversion, string concatenation, and/or
linear length constraints over variables used within the regular
expression membership predicate. The benchmarks contained combinations of these
operations. The goal was now to group them into different first order logic
theories, which will be introduced in the next section.

\smallskip

\noindent \textbf{The resulting first order logic theories.} The basis of the following
theories is built by $\folstruc{}_s$, the theory of simple regular expressions.
While categorising the benchmarks, we identified four important, (partially)
disjoint theories, forming extensions of the aforementioned theory. The
vocabulary of \emph{extended regular expressions} is given by $\mathcal{R}_e =
\mathcal{R}_s \cup \Set{\overline{\phantom{a}}/\!\!/1}$. In principle, it adds
the complement to our basis. The many-sorted $\mathcal{R}_e$-structure
$\mathcal{A}_{e} = \mathcal{A}_{e} \cup
\Set{\overline{\phantom{a}}^{{\regl{\terminals}{\variables}}}}$ therefore simply
adds the complement having the semantics $L(\overline{R_1}) = L(R_1^*) \setminus
L(R_1)$ to the theory $\folstruc{}_s$. Let
$\regcl{\terminals{}}{}$ denote the set of all regular expressions including complement,
inductively defined as seen above. \looseness=-1

Furthermore, in practice solutions to variables are often restricted by linear
inequalities ranging over the length of potential solutions. Therefore a natural
extension is adding a function to our our vocabularies allowing us to reason
about length. Let $\mathcal{R}_{il} = \mathcal{R}_i \cup
\Set{\mathds{Z},+/\!\!/2,\leq/2,\dot{0},\len/\!\!/1}$ be a vocabulary where $i
\in \Set{e,s}$, being characterised by previously defined axioms and
additionally the associativity and commutativity of $+/\!\!/2$, the existence of a
neutral element, and the requirement that $\leq$ be a total ordering and
monotonic on our domain. The many-sorted $\mathcal{R}_{il}$-structure of
\emph{regular expressions with length} is defined by $\mathcal{A}_{\mathrm{il}}
= \mathcal{A}_i \cup
\{\;+^\mathds{Z},\leq^\mathds{Z},\dot{0}^\mathds{Z},\len^{\terminals{}
\rightarrow \mathds{Z}}\,\},$ where $+^\mathds{Z}$,$\leq^\mathds{Z}$ are defined
as commonly used operations over $\mathds{Z}$, $\dot{0}^\mathds{Z} = 0 \in
\mathds{Z}$, and the length function $\len^{\terminals{} \rightarrow
\mathds{Z}}$ for a pattern $\alpha \in \pats{\terminals{}}{}$ and an assignment
$\substitution{} \in \assignments{\terminals{}}$ by $\len^{\terminals{}
\rightarrow \mathds{Z}}(\alpha) = |\substitution{}(\alpha)|$.

A third addition often occurring in real-world program analysis is a
string-number conversion predicate. To this extend let $\mathcal{R}_{in} =
\mathcal{R}_i \cup \Set{\numstr/2}$ whereas $i \in \Set{e,s,el,sl}$ be a
vocabulary. The axioms are derived from the corresponding base theory. The
many-sorted $\mathcal{R}_{in}$-structure of \emph{regular expressions with
number conversation} is defined by $\mathcal{A}_{\mathrm{in}} = \mathcal{A}_i
\cup \Set{\mathds{N},\numstr^{\mathds{N}\,\terminals^*}}$, whereas
$\numstr^{\mathds{N}\,\terminals^*}$ is a relation, which holds for all positive
integers $i \in \mathds{N}$ and words $w \in \Set{0,1}^*$ where $w$ -- possibly
having leading zeros -- is the binary representation of $i$, formally defined by
$\numstr(n,w) \text{ iff } w \rein (0 \cup 1)^* \land n \geq 0 \land \sum_{j \in
\Set{1,\dots,|w|}} w[j] \cdot 2^{|w| - j}.$

Naturally, not only in real-world applications, it is interesting to ask whether
a pattern $\alpha \in \pats{\terminals{}}{arg}$ possibly containing variables is
bound by a regular language. This leads to the last extension we are considering
in this work.  Let $\mathcal{R}_{ic} = \mathcal{R}_i \cup \Set{\cdot/\!\!/2}$
whereas $i \in \Set{e,s,el,sl,eln,sln,en,sn}$ be a vocabulary, having the
additional axioms induced by
$(\pats{\terminals{}}{},\cdot^{\terminals{}},\epsilon)$ forming a monoid. The
many-sorted $\mathcal{R}_{ic}$-structure of \emph{regular expressions with
concatenation} is defined by $\mathcal{A}_{\mathrm{ic}} = \mathcal{A}_i \cup
\Set{\cdot^{\terminals{}},\dot{\epsilon}^\terminals{}},$ whereas
$\cdot^{\terminals{}}$ is defined as the classical concatenation over
$\pats{\terminals{}}{}$ and $\dot{\epsilon}^\terminals{} = \epsilon \in
\terminals^*$. \looseness=-1
These theories are again naturally combined with the theory of word equations by
simply considering the union of their components.

As an example, consider the string constraint $C = \var x_1 \rein 1^*  \land
\numstr{}(15,\var x_1) \land \len(\var x_1) \geq 3$ where $x_1 \in \variables$
and $1 \in \terminals{}$. A solution $\substitution{} \in
\assignments{\terminals{}}$ is given by $\substitution{}(\var x_1) = 1111$,
since $\substitution{}(\var x_1) = 1111 \in L(1^*)$, $\numstr{}(15, 1111)$
because $1111$ is the binary representation of $15$, and $\substitution{}(\var
x_1) \geq 3$. Therefore $\folstruc{}_{sln},\substitution{} \models C$.

\smallskip

\noindent \textbf{Benchmark analysis}
\begin{figure}[t!]
\centering


\resizebox{.95\columnwidth}{!}{
\begin{tikzpicture}[-,>=stealth',shorten >=1pt,auto,node distance=1cm,
semithick]

\tikzset{every node/.append style={font=\Huge}}

\node[]   (Ws)  at (1.4065,3) []           {$\mathcal{A}^{\doteq}_{s}$};
\draw [unit,fill=black!5] (0.0,0) rectangle (2.813,1);
\path (Ws) edge [-] node [] {}   (1.4065,0.5);
\node[]   (Wsl)  at (14.115,3) []           {$\mathcal{A}^{\doteq}_{sl}$};
\draw [unit,fill=black!15] (2.813,0) rectangle (25.417,1);
\path (Wsl) edge [-] node [] {}   (14.115,0.5);
\node[]   (Wsn)  at (24.235,3) []           {$\mathcal{A}^{\doteq}_{sn}$};
\draw [unit,fill=black!25] (25.417,0) rectangle (26.053,1);
\path (Wsn) edge [-] node [] {}   (25.735,0.5);
\node[]   (Wslc)  at (25.8955,3) []           {$\mathcal{A}^{\doteq}_{slc}$};
\draw [unit,fill=black!35] (26.053,0) rectangle (26.138,1);
\path (Wslc) edge [-] node [] {}   (26.0955,0.5);
\node[]   (Wsln)  at (27.6915,3) []           {$\mathcal{A}^{\doteq}_{sln}$};
\draw [unit,fill=black!45] (26.138,0) rectangle (26.445,1);
\path (Wsln) edge [-] node [] {}   (26.2915,0.5);
\node[]   (Wslnc)  at (29.908,3) []           {$\mathcal{A}^{\doteq}_{slnc}$};
\draw [unit,fill=black!55] (26.445,0) rectangle (26.451,1);
\path (Wslnc) edge [-] node [] {}   (26.448,0.5);
\node[]   (Welc)  at (32.452,3) []           {$\mathcal{A}^{\doteq}_{elc}$};
\draw [unit,fill=black!65] (26.451,0) rectangle (26.453,1);
\path (Welc) edge [-] node [] {}   (26.452,0.5);

\node[]   (s)  at (12.128,2) []           {$\folstruc_{s}$};
\draw [unit,fill=black!5] (0.0,-1.5) rectangle (24.256,-0.5);
\path (s) edge [-] node [] {}   (12.128,-1.0);
\node[]   (sc)  at (22.259,2) []           {$\folstruc_{sc}$};
\draw [unit,fill=black!15] (24.256,-1.5) rectangle (24.862,-0.5);
\path (sc) edge [-] node [] {}   (24.559,-1.0);
\node[fill=white]   (sl)  at (27.5255,0.5) []           {$\folstruc_{sl}$};
\draw [unit,fill=black!25] (24.862,-1.5) rectangle (29.189,-0.5);
\path (sl) edge [-] node [] {}   (27.0255,-1.0);
\node[]   (sn)  at (29.029,0.5) []           {$\folstruc_{sn}$};
\draw [unit,fill=black!35] (29.189,-1.5) rectangle (29.869,-0.5);
\path (sn) edge [-] node [] {}   (29.529,-1.0);
\node[]   (slc)  at (30.7005,0.5) []           {$\folstruc_{slc}$};
\draw [unit,fill=black!45] (29.869,-1.5) rectangle (30.332,-0.5);
\path (slc) edge [-] node [] {}   (30.1005,-1.0);
\node[]   (sln)  at (32.5325,0.5) []           {$\folstruc_{sln}$};
\draw [unit,fill=black!55] (30.332,-1.5) rectangle (30.533,-0.5);
\path (sln) edge [-] node [] {}   (30.4325,-1.0);
\node[]   (e)  at (33.4365,-0.5) []           {$\folstruc_{e}$};
\draw [unit,fill=black!65] (30.533,-1.5) rectangle (30.54,-0.5);
\path (e) edge [-] node [] {}   (30.5365,-1.0);

\tikzset{every node/.append style={font=\huge}}

\draw[black,-] (-0.5,-1.85) -- (31.04,-1.85);
\draw[black,-] (0.0,-1.75) -- (0.0,-1.85);
\node[]   (I0)  at (0.0,-2.3) []           {0};
\draw[black,-] (3.054,-1.75) -- (3.054,-1.85);
\node[]   (I3054.0)  at (3.054,-2.3) []           {3054};
\draw[black,-] (6.108,-1.75) -- (6.108,-1.85);
\node[]   (I6108.0)  at (6.108,-2.3) []           {6108};
\draw[black,-] (9.162,-1.75) -- (9.162,-1.85);
\node[]   (I9162.0)  at (9.162,-2.3) []           {9162};
\draw[black,-] (12.216,-1.75) -- (12.216,-1.85);
\node[]   (I12216.0)  at (12.216,-2.3) []           {12216};
\draw[black,-] (15.27,-1.75) -- (15.27,-1.85);
\node[]   (I15270.0)  at (15.27,-2.3) []           {15270};
\draw[black,-] (18.324,-1.75) -- (18.324,-1.85);
\node[]   (I18324.0)  at (18.324,-2.3) []           {18324};
\draw[black,-] (21.378,-1.75) -- (21.378,-1.85);
\node[]   (I21378.0)  at (21.378,-2.3) []           {21378};
\draw[black,-] (24.432,-1.75) -- (24.432,-1.85);
\node[]   (I24432.0)  at (24.432,-2.3) []           {24432};
\draw[black,-] (27.486,-1.75) -- (27.486,-1.85);
\node[]   (I27486.0)  at (27.486,-2.3) []           {27486};
\draw[black,-] (30.54,-1.75) -- (30.54,-1.85);
\node[]   (I30540.0)  at (30.54,-2.3) []           {30540};

\node[font=\huge]   (title)  at (-0.65,0.5) []           {(a)};
\node[font=\huge]   (title)  at (-0.65,-1) []           {(b)};

\end{tikzpicture}}
\caption{Distribution of instances among their theories. (a) instances with word equations (b) instances without word equations.}
\label{fig:instances}
\vspace*{-0.5cm}
\end{figure}
The analysis of the 56993 instances reveals that 30540 instances are solely a member of one of our regular expression theories, while 26453 additionally contained word equations. In Figure~\ref{fig:instances} we plot the distribution of all instances w.r.t. their theory. We display the instances according to the presence of word equations into two bars (a) and (b). The width of a single block within a bar corresponds to the instance count of the smallest theory. Since some of the theories are disjoint (e.g. $\folstruc{}_{sl}$ and $\folstruc{}_{sn}$) the diagram does not visualise inclusions.\looseness=-1

Within the pure regex formulae, the most frequented theory is $\folstruc{}_s$ holding 24256 instances. As we will see in this work, this theory and also its successor $\folstruc{}_{sl}$ with 4327 instances are $\mathsf{PSPACE}$-complete and raises hope for efficient solving strategies. The theories $\folstruc_{elnc}$ and $\folstruc_{slnc}$, for which we prove undecidability within this work, do not seem to have a high relevance in application since they do not occur at all within our analysed set of benchmarks.

On the other hand, the instances containing word equations are also based around simple regular expressions. The most prominent theory is $\mathcal{A}^{\doteq}_{sl}$ holding 22604 instances, followed by $\mathcal{A}^{\doteq}_{s}$ containing 2813 instances. Unfortunately, the decidability of the largest set of instances is not known. Notably, the total set only contains 9 instances based on the theory $\folstruc{}_{e}$ where the complement is actually needed. All other instances can be rewritten to simply avoid the complement.

\section{Decidability of the Theories}
\label{sec:identification}
In this section\footnote{All proofs are omitted due to space constraints and can be found in Appendix~\ref{sec:proofs}.}, we characterise the related quantifier-free first-order
theories introduced in Section \ref{sec:prelim} according to their decidability.
The contributions are summarized in Figure~\ref{fig:decidability}. The arrows
lead from stronger and more expressive theories to weaker ones. Theories in the
upper box are undecidable, while those in the lower box are decidable (similarly,
the theories within the inner dashed box are $\mathsf{PSPACE}$-complete). We
proceed with a summary of the theorems we prove and some discussion of the
motivation and intuition for the proofs.

In an attempt to move from simpler to more complicated theories, we will begin
our journey with the theory without complement operation for
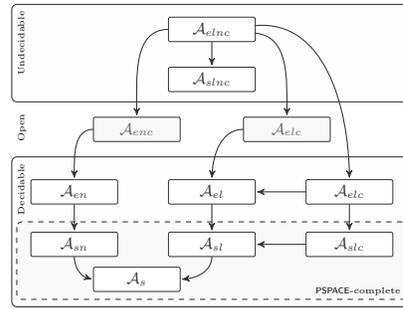
\begin{wrapfigure}[16]{r}{0.5\textwidth}
\centering
\vspace*{-0.5cm}
\resizebox{.45\textwidth}{!}{
\begin{tikzpicture}[->,>=stealth',shorten >=1pt,auto,node distance=1cm,
semithick]

\tikzset{every edge/.append style={font=\tiny,color=tbg}}
\tikzset{every unit/.append style={font=\Huge}}

\draw[tbg,rounded corners=3pt,fill=tbg!1] (-4,0.5)  rectangle (4,-1.45);
\draw[tbg,rounded corners=3pt,fill=tbg!1] (-4,-2.55)  rectangle (4,-5.55);
\draw[tbg,rounded corners=3pt,fill=tbg!3,dashed] (-3.9,-3.85)  rectangle (3.9,-5.4);
\node[anchor=north east,rotate=90,font=\tiny] at (-4,0.5)              {Undecidable};
\node[anchor=north east,rotate=90,font=\tiny] at (-4,-1.55)              {Open};
\node[anchor=north east,rotate=90,font=\tiny] at (-4,-2.55)              {Decidable};
\node[anchor=south east,rotate=0,font=\tiny] at (3.9,-5.45)              {$\mathsf{PSPACE}$-complete};

\node[unit,font=\small]   (0)  at (0,0)    {$\folstruc_{elnc}$};	
\node[unit,below of=0,font=\small]   (1)              {$\folstruc_{slnc}$};

\node[unit,fill=black!3,text=black!60,font=\small]   (2a)  at ($(1)+(-1.5,-1)$)              {$\folstruc_{enc}$};
\node[unit,fill=black!3,text=black!60,font=\small]   (3a) at ($(1)+(1.5,-1)$)             {$\folstruc_{elc}$};

\node[unit,font=\small]   (2)  at ($(1)+(-2.75,-2.25)$)              {$\folstruc_{en}$};
\node[unit,font=\small]   (3) at ($(1)+(2.75,-2.25)$)             {$\folstruc_{elc}$};
\node[unit,below of=1, node distance=2.25cm,font=\small]   (4)              {$\folstruc_{el}$};

\node[unit,font=\small]   (5) at ($(4)+(-2.75,-1.05)$)             {$\folstruc_{sn}$};
\node[unit,font=\small]   (6) at ($(4)+(2.75,-1.05)$)             {$\folstruc_{slc}$};
\node[unit,below of=4, node distance=1.05cm,font=\small]   (7)              {$\folstruc_{sl}$};
\node[unit,font=\small]   (8) at ($(4)+(-1.5,-1.75)$)             {$\folstruc_{s}$};

\path (0) edge [] node [] {}   (1)
		  edge [out=180,in=90] node [] {}   (2a)
		  edge [out=-5,in=90] node [] {}   (3a)
		  edge [out=5,in=90] node [] {}   (3)
	  (2a) edge [out=180,in=90] node [] {}   (2)	
	  (3a) edge [out=180,in=90] node [] {}   (4)
	  (2) edge [] node [] {}   (5)	
	  (3) edge [] node [] {}   (6)
		  edge [] node [] {}   (4)
	  (4) edge [] node [] {}   (7)
	  (5) edge [out=-90,in=180] node [] {}   (8)
	  (6) edge [] node [] {}   (7)
	  (7) edge [out=-90,in=0] node [] {}   (8)
;

\end{tikzpicture}
}
\caption{Visualization of relationship and decidability of various extensions of $\folstruc_{s}$, with arrows leading from stronger theories to theories which they contain.}
\label{fig:decidability}
\end{wrapfigure}
regular expressions. We will start be considering $\folstruc_{slc}$. The
motivation in approaching this theory first (formalized later in Theorem
\ref{thm:lobound}) is that for more general theories, which include regular
expressions with complement operations, even simple tasks (like checking whether
there exists a common string in the languages of two given expressions) require
an exponential amount of space. One way to understand this is that the
exponential blow-up with respect to the size of the regular expressions comes
from transforming this expression into an NFA, determinising it, and then
computing its complement. In fact, we will see that any other approach
inherently leads to such an exponential blow-up. We can state the following
result.

\begin{restatable}{theorem}{lsre}
\label{thm:lsre}
The satisfiability problems for $\folstruc_{slc}$ and $\folstruc_{sl}$ of simple regexes, linear integer arithmetic, string length, and concatenation
are $\mathsf{PSPACE}$-complete. 
\end{restatable}

\noindent 


In the following we sketch the proof of this theorem. It is enough to show the
statement for $\folstruc_{slc}$. Assume that the input is a formula $\varphi$.
We first  propagate all negations top-down in the formula, so that we obtain an
equivalent formula $\varphi'$ which consists of a Boolean combination of atoms
of the form $\alpha\rein R$ or $\neg(\alpha\rein R)$, where
$\alpha\in\pats{\terminals{}}{arg}$ and $R\in\regl{\terminals{}}{}$, as well as
atoms encoding arithmetic constraints. Clearly, $|\varphi'|\in
\mathcal{O}(|\varphi|)$. Then, we non-deterministically choose an assignment of
truth values for all atoms such that $\varphi'$ evaluates to true. As such, we
get from our formula a list ${\mathcal L}_r$ of atoms of the form $\alpha\rein
R$ or $\neg(\alpha\rein R)$, where $\alpha\in\pats{\terminals{}}{arg}$ and
$R\in\regl{\terminals{}}{}$, that have to be true. If an atom $\alpha\in R$ is
false in our assignment, then ${\mathcal L}_r$ will contain $\neg(\alpha\rein
R)$, and if $\neg(\alpha\rein R)$ is false in the assignment, then ${\mathcal
L}_r$ will contain $\alpha\rein R$. We similarly construct a second list
${\mathcal L}_{n}$ with the arithmetic linear constraints that should be true.
Clearly, an assignment of the variables occurring in these two lists such that
all the atoms they contain are evaluated to true exists if and only if $\varphi$
is satisfiable.

Let us first neglect the polynomial space requirement. We construct the NFA
$M_R$ for each regex $R\in\regl{\terminals}{}$ occurring in ${\mathcal L}_r$.
Following a folklore  automata-theoretical approach (reminiscent of classical
algorithms converting finite automata into regexes and vice versa, and also used
in string solving in, e.g., \cite{norn,AbdullaACHRRS15}), each occurrence of a
variable $\var x \in \variables$ should be assigned a path in one of the NFAs $M_R$
(if $\var x \in \vars{\alpha}$ for an atom $\alpha\rein R$) or in
the NFA $\overline{M_R}$,  accepting the complement of the language accepted by
$M_R$ (if $\var x \in \vars{\alpha}$ for an atom $\neg(\alpha\rein R)$). 
This assignment should be {\em correct}: for each atom
$\alpha\rein R$ (resp., $\neg(\alpha\rein R)$), concatenating the paths assigned
to the occurrences of the variables of $\alpha$, in the order in which they
occur in $\alpha$, we should get an accepting path in $M_R$ (resp.,
$\overline{M_R}$). Hence, it is enough to associate to each occurrence of each
variable the starting and ending state of the respective paths, and then ensure
that we connect these states by the same word for all occurrences of the same
variable. That is, we associate to an occurrence of a string variable $\var
x\in\variables{}$ occurring in $\alpha\rein R$ a copy of the automaton $M_R$
with the initial and final state changed, so that they correspond to the
starting and ending state on the path of $M_R$ associated to the respective
occurrence of $\var x$ (and similarly for $\overline{M_R}$). So, if $\var x_i$
is the $i^{th}$ occurrence of $\var x$ in $\varphi$, then we associate an NFA
$M_{\var x,i}$ to it. We intersect all the automata $M_{\var x,i}$ to obtain an
NFA $A_{\var x}$ which accepts exactly those strings which are a {\em correct}
assignment for the variable $\var x$. 

Observe now that if a word is accepted in $A_{\var x}$ then its length is part
of an arithmetic progression, from a finite set of arithmetic progressions
\cite{Chrobak86,Gawrychowski11}. Conversely, each element of these arithmetic
progressions is the length of a word accepted by $A_{\var x}$, and the set of
progressions can be computed based only on the underlying graph of the NFA
$A_{\var x}$. Hence, we get several new linear arithmetic constraints on the
length of our variables, which are satisfied if and only if there exists a {\em
correct} assignment for the variables. We add this new set of constraints to
${\mathcal L}_n$ and then solve the resulting linear integer system with
standard methods. 

Finally, if, and only if, the final set of linear constraints we defined is
satisfiable, then $\varphi'$ and, consequently, $\varphi$ are also satisfiable. 

This ends the description of our decision procedure, which is based on
relatively standard automata-theory techniques. To show the
$\mathsf{PSPACE}$-membership we use the fact that the regexes of
$\folstruc_{slc}$ do not have complements. Firstly, note that we can just build
the NFAs for all the regexes occurring in the positive or negative atoms
$\varphi'$ (and not complement any of them). Once these automata are built, we
do not have to explicitly construct the automata $M_{\var x,i}$ or $A_{\var x}$:
we implicitly know their states and the transitions that may occur between them.
Indeed, the states are tuples of states of the original NFAs $M_R$, and, as we
do not have complements in any expression $R$, the number of components in each
tuple is bounded by a polynomial in the size of $\varphi$; the transitions
between such states can be simulated by looking at the transitions of the
original NFAs. Computing (and storing) the linear constraints on the length of
the {\em correct} assignments for $\var x$ from $A_{\var x}$ can also be done in
polynomial space (because of the bounds on the number of states of the automata
$A_{\var x}$). We obtain, as such, a system of linear arithmetic constraints
with coefficients of polynomial size (w.r.t. the size of $\varphi$). Thus,
solving the derived system can be done in polynomial space.

The lower bounds stated in Theorem \ref{thm:lsre} follow immediately from the
$\mathsf{PSPACE}$-completeness of the intersection problem for NFAs.


When we allow arbitrary complements in the regular expressions, we can still
prove the decidability of the respective theories but the complexity
increases.\looseness=-1
\begin{restatable}{theorem}{lrec}\label{thm:lrec} The satisfiability
problems for $\folstruc_{elc}$ and $\folstruc_{el}$ of regular expressions,
linear integer arithmetic, concatenation, and string length are decidable.
\end{restatable} The idea is to use the same strategy as explained above for
$\folstruc_{slc}$. Since regular expressions may now contain complements, when
constructing the automaton $M_R$ associated with a regex $R \in
\regl{\terminals{}}{}$ we might have an exponential blow-up in size, even if the
alphabet of the regex (resp. NFA) is binary and only one complement is used (as
shown, for instance, in \cite{HospodarJM18}). We can no longer guarantee the
polynomial space complexity of our approach, but the decidability result
holds.\looseness=-1

This theorem is supplemented by the following remark, which shows upper and,
more interestingly, lower bounds for the space needed to decide the
satisfiability problem for a formula in the quantifier-free theories
$\folstruc_{el}$ and $\folstruc_{elc}$.
\begin{remark}\label{rem:blow_up}
Let $g : \mathds{N}_{>0} \times \mathds{Q} \rightarrow \mathds{Q}$ recursively defined by $g(1,c)=2^c$ and $g(k+1,c)=2^{g(k,c)}$ for $k\in\mathds{N}_{> 0}$ and $c \in \mathds{Q}$. Informally this mapping corresponds to the following tower of powers (a.k.a. tetration) $g(k,c)=
{\underbrace{{2
  {{{^{2\vphantom{h}}}^{2\vphantom{h}}}^{\dots\vphantom{h}}}^{2\vphantom{h}}}
}_{\text{$k$ times}}}^c = 2\uparrow^k c$.


For a regex $R\in\regcl{\terminals{}}{}$, define the complement-depth
$\compdepth : \regcl{\terminals{}}{} \rightarrow \mathds{N}$ recursively as
follows. If $R \in \Set{\emptyset,\epsilon,a}$ for $a \in \terminals{}$ let
$\compdepth(R)=0$. Otherwise if $R \in \Set{R_1 \cup R_2, R_1 \cdot R_2}$ let
$\compdepth(R)=\compdepth(R_1)+\compdepth(R_2)$, if $R = R_1^*$ let
$\compdepth(R)=\compdepth(R_1)$, and if $R = \overline{R_1}$ let
$\compdepth(R)=1+\compdepth(R_1)$ for appropriate $R_1,
R_2\in\regcl{\terminals{}}{}$. For a formula $\varphi$ in the quantifier-free
theory $\folstruc_{elc}$ (as well as  $\folstruc_{el}$) we let
$\compdepth(\varphi)$ be the maximum depth of a regex in $\varphi$.

One can show, using for instance our approach from the proofs of Theorems
\ref{thm:lsre} and \ref{thm:lrec}, that the satisfiability problem for formulae
$\varphi$ from the quantifier-free theory $\folstruc_{elc}$ (and
$\folstruc_{el}$ as well), with size $n\in\mathds{N}$ and
$\compdepth(\varphi)=k\in\mathds{N}$, is in $\mathsf{NSPACE}$$(f(g(k-1,2n)))$,
where $f$ is a polynomial function. However, there exists a positive rational
number $c\in\mathds{Q}$ such that the respective problem is not contained in
$\mathsf{NSPACE}$$(g(k-1,cn))$. This lower bound follows from \cite{stockmeyer}.
There, the following problem is considered: Given a regex
$R\in\regcl{\terminals{}}{}$, of length $n$, with $\compdepth(R)=k \in
\mathds{N}$ over an alphabet $\terminals{}$, decide whether $L(R)=\terminals^*$.
It is shown that there exists a positive rational number $c$ such that the
respective problem cannot be solved in $\mathsf{NSPACE}$$(g(k,cn))$. So,
deciding whether a formula $\varphi$ of $\folstruc_{el}$ consisting of the atoms
$\alpha\rein \overline{R}$ and $\alpha\in \terminals^*$, where
$R\in\regcl{\terminals{}}{}$ is a regex of length $n$ with $\compdepth(R)=k-1$,
is not contained in $\mathsf{NSPACE}$$(g(k-1,cn))$ (note that, in this case, the
length of the formula $\varphi$ is also $O(n)$). 

Intuitively, this lower bound shows that if the complement-depth of a formula of
length $n$ is $k$, then checking its satisfiability inherently requires an
amount of space proportional to the value of the exponentiation tower of height
$k-1$, and with the highest exponent $cn$.  \hfill{$\triangleleft$} \end{remark}


Clearly, the satisfiability problem for the quantifier-free theory $\folstruc_{el}$ is also decidable according to the theorem above. 
Let $g$ be defined as given in Remark~\ref{rem:blow_up}. Based on the classical results from \cite{stockmeyer}, we can derive the following theorem.
\begin{restatable}{theorem}{lobound}\label{thm:lobound}
There exists a positive rational number $c$ such that the satisfiability problem for the fragments of $\folstruc_{s}$ and $\folstruc_{sc}$ allowing only formulae of complement-depth at least $k$ is not in $\mathsf{NSPACE}$$(g(k-1,cn))$.
\end{restatable}
This theorem shows that, in fact, when deciding the satisfiability problem for
the quantifier-free theories $\folstruc_{elc}$ and $\folstruc_{el}$ the
automata-based proof we presented is
relatively close to the space-complexity lower bound for this problem. Any other
approach, automata-based or otherwise, would still meet the same obstacle: the
space complexity of any algorithm deciding the satisfiability of formulae of
complement-depth $k$ cannot go under the $\mathsf{NSPACE}$$(g(k-1,cn))$ bound.
This, on the one hand, explains our interest in analysing the theory
$\folstruc_{sl}$ (and its variants): as soon as we consider stacked complements, we are out of
the $\mathsf{PSPACE}$ complexity class. On the other hand, this also explains
the reason why in developing a practical solution for the satisfiability problem
of $\folstruc_{el}$ formulae within our tool \toolname{} we use many heuristics.
While the result of Theorem \ref{thm:lrec} was known from \cite{CVC4-FROCOS15},
our approach seems to provide a deeper understanding of the hardness of this
problem (and where this stems from) and of the ways we can deal with~it.

Next we consider the case when we replace the length function by the $\numstr$
predicate. 
The lower bound of Theorem \ref{thm:lobound} applies also to the case of
$\folstruc_{en}$. So one cannot hope to solve the satisfiability problem for
this theory in polynomial space, as soon as we allow arbitrary complements in
our regular expressions. However, we can show that the satisfiability problem
for $\folstruc_{en}$ is decidable, and in $\mathsf{PSPACE}$ when only simple regular expressions are
allowed.
\begin{restatable}{theorem}{lsren}\label{thm:sren} The satisfiability problem
for $\folstruc_{sn}$ (resp. for $\folstruc_{en}$) of (simple) regexes and a string-number predicate is $\mathsf{PSPACE}$-complete
(resp. decidable). \end{restatable} 

While the general idea to prove the above result is based on a similar
construction to that in Theorem~\ref{thm:lsre}, in this case we need to use a
different strategy to work with the linear arithmetic constraints (due to the
fact that $\numstr$ predicates are involved, and their fundamentally different
nature w.r.t. the length function). More precisely, we use the fact that
deciding whether the set of linear constraints is satisfiable is equivalent to
checking whether the language accepted by a finite synchronized multi-tape
automaton is empty (see \cite{GaneshBD02}); this framework allows us to
canonically integrate regular constraints. The entire approach is now
automata-based and, once again, the key to showing the $\mathsf{PSPACE}$
membership is the fact that these automata can be simulated in polynomial space.

It is natural to ask whether the decidability result extends to the theories
$\folstruc_{enc}$ (and $\folstruc_{snc}$), which also allow concatenation. While
we leave this open, one can make two interesting observations. Firstly, these
theories are expressive enough to define a predicate checking if two strings
have equal length.  Moreover, $\folstruc_{enc}$ (and likewise $\folstruc_{snc}$)
has equivalent  expressive power to the theory of word equations with regular
constraints, a predicate allowing the comparison of the length of string terms,
and the $\numstr$ predicate. The decidability of word equations with
string-length comparisons is a long standing open problem, so we also consider
it worthwhile to address the decidability of the slightly stronger theory
$\folstruc_{enc}$. According to~\cite{rp2018strings}, the theory of word
equations, length constraints, and string-number conversion (modelled by the
$\numstr$ predicate) is undecidable; the difference is that in that theory, one
can check whether the length of a string term equals an integer term, which
seems more general than what one can model in $\folstruc_{enc}$. We get the
following. \looseness=-1
\begin{restatable}{theorem}{decfour} \label{thm:undecWE} The
satisfiability problem for $\folstruc_{slnc}$ of regular expressions, linear
integer arithmetic, a string-number predicate and concatenation is
undecidable. 
\end{restatable} 

In conclusion, $\folstruc_{enc}$ and $\folstruc_{eln}$ are the only fragments of
$\folstruc_{elnc}$ where the decidability status of the satisfiability problem
remains open.

\section{Conclusion}
\label{sec:conclusion}
Within this work we analysed \totalinstances{} string solving benchmarks 
containing regular expression membership queries and identified relevant sub-theories based around
regular membership predicates. It turned out that the most frequently occurring sub-theory is decidable. 
Notably, the ideas of these proofs directly lead to a well-performing solver for regular expression
membership predicates. 
This paper also shows that an interleaving between theory and practice potentially leads to new 
interesting solutions in both worlds. Our future work will continue on this trail to obtaining relevant
sub-theories used in practice, always in the hope of finding decidable sub-theories which lead to the 
design of new decision procedures for solving practically relevant string constraints.

\bibliographystyle{splncs04}

\begin{btSect}{words}
\section*{References}
\btPrintAll
\end{btSect}

\newpage
\appendix
\setcounter{theorem}{0}
\section{Proofs}
\label{sec:proofs}
\begin{lemma}
\label{lem:slc-pspace}
The satisfiability problems for $\folstruc_{slc}$ and $\folstruc_{sl}$ of simple regular expressions, linear integer arithmetic, string length, and concatenation
are decidable in $\mathsf{PSPACE}$. 
\end{lemma}
\begin{proof}
We first show this result for $\folstruc_{slc}$.  
Consider a formula $\varphi$ from $\folstruc_{slc}$. We will give a non-deterministic algorithm that decides whether $\varphi$ is satisfiable in polynomial space. However, for a simpler presentation, we first present an algorithm deciding the satisfiability of $\varphi$ without the space restriction. 

Firstly, we convert the formula $\varphi$ into an equivalent formula $\varphi'$ in negation normal form. Therefore $\varphi'$ consists only of a Boolean combination ($\vee$ and $\wedge$) of atoms of the form $\alpha\rein R$ or $\neg(\alpha\rein R)$, where $\alpha\in\pats{\terminals{}}{}$ is a string term and $R \in \regl{\terminals{}}{}$, as well as atoms encoding arithmetic constraints. Clearly, $|\varphi'|\in \mathcal{O}(|\varphi|)$. 

Secondly, we non-deterministically choose an assignment of truth values for all atoms such that the Boolean abstraction of $\varphi'$ is satisfiable. As such, we get from our formula a list ${\mathcal L}_{r}$ of atoms of the form $\alpha\rein R$ or $\neg(\alpha\rein R)$, where $\alpha \in \pats{\terminals{}}{}$ and $R \in \regl{\terminals{}}{}$, that have to evaluate to true; if an atom $\alpha\rein R$ was assigned false in the assignment we chose, then we put in the list $\neg(\alpha\rein R)$, and if $\neg(\alpha\rein R)$ was assigned false in our assignment, then we put in the list $\alpha\rein R$, while all the atoms that are assigned true are added to the list as they are. We similarly construct a second list ${\mathcal L}_{n}$ containing a set of arithmetic linear constraints that should be evaluated to true. If, and only if, we find an assignment of the variables occurring in these two lists such that all the atoms they contain are evaluated to true, then $\varphi$ will be satisfiable.
 
Thirdly, if $\alpha\rein R$ is a regular membership constraint of ${\mathcal L}_{r}$, let $M_R$ be the NFA accepting the language of $R$; the constraint $\alpha\rein R$ is equivalent to $\alpha\in L(M_R)$. If $\neg(\alpha\rein R)$ is in ${\mathcal L}_{r}$, let $\overline{M}_R$ be the NFA accepting the language of $\overline{R}$;  the constraint $\neg(\alpha\rein R)$ is equivalent to $\alpha\in L(\overline{M}_R)$. Essentially, the list ${\mathcal L}_{r}$ can be seen as a list of constraints $\alpha\in L(M)$, where $\alpha\in\pats{\terminals{}}{}$ and $M$ is an NFA. Without loss of generality, we assume each of the NFAs appearing in ${\mathcal L}_r$ has exactly one initial state, one final state, and no $\epsilon$-transitions.

Now, consider the constraint $\alpha\in L(M)$ for $M = (Q,\terminals{},\delta,q_0,\Set{f})$ from our list, and note that $\alpha$ is either a single string variable or the concatenation of several string variables. It is clear that $\models \alpha\in L(M)$ there is a way of building an assignment $\substitution{} \in \assignments{\terminals{}}$ such that for each variable $\var x \in \vars{\alpha}$ there exists a path $\substitution(\var x)$ in $M$ and the entire pattern $\substitution(\alpha)$ forms a path leading from $q_0$ to $f$ in $M$. Therefore $\substitution{}(\alpha) \in L(M)$.

Let $\alpha = \var x_1 \dots \var x_k$ for $k \in \mathds{N}$ and $\var x_i \in \vars{\alpha}$. We non-deterministically choose a starting state $q_{\var x,i} \in Q$ and an ending state $f_{\var x,i} \in Q$ for each occurrence $i \leq k$ of each variable $\var x \in \vars{\alpha}$, such that there exists a $w_{\var x,i} \in \terminals{}$ and $\delta(q_{x,i},w) \subseteq \Set{f_{x,i}}$ and $\delta(q_0,w_{\var x_1,1} \dots w_{\var x_k,i)} \subseteq \Set{f}$ for $i = |\alpha|_{\var x_k}$.

Each variable $\var x \in \vars{\alpha}$ must have an assignment that is accepted by all NFAs $M_{\var x,i}$, constructed for each of its occurrences $i \leq \ell$ from each constraint $\alpha\in L(M)$. Hence, we intersect all NFAs $M_{\var x,i}$ for all $\var x$ and all $i \in \Set{1, \dots |\alpha|_{\var x}}$ and get a new NFA $A_{\var x} = (Q_{\var x},\terminals{},\delta_{\var x},q_{0_{\var x}},F_{\var x})$. 

Further, let $B_{\var x} = (Q_{\var x},\Set{a},\delta'_{\var x},q_{0_{\var x}},F_{\var x})$ be the unary NFA obtained by re-labelling all transitions in $A_{\var x}$ with a single letter $a$, namely $\delta'_{\var x} = \{(q,a) \mapsto p\;|\;q,p \in Q, b \in \terminals{}, \delta(q,b) \subseteq \Set{p} \}$. It is clear that the paths of $A_{\var x}$ correspond bijectively to the paths of $B_{\var x}$. Let $m = |Q_{\var x}|$ be the number of states of $A_{\var x}$ and $B_{\var x}$. A well known result, related to the Chrobak normal form of unary automata \cite{Chrobak86}, is that {\em all accepting paths} in $B_{\var x}$ that go through a state $q \in Q_{\var x}$ from the initial state $q_{0_{\var x}}$ of $B_{\var x}$ to the final state $f_{\var x}$ of $B_{\var x}$ can be succinctly represented as the shortest path from $q_{0_{\var x}}$ going through $q$ to $f_{\var x}$, whose length is $\ell^q_p \leq 2m$, and the shortest cycle containing $q$, whose length is $\ell^q_c \leq m$ (see, for instance, the statement and proof of Lemma~1 of \cite{Gawrychowski11}). Thus, for each node $q$ of $B_{\var x}$ (or, equivalently, $A_{\var x}$), we can find the smallest $\ell^q_c\leq 2m$ such that there is a path from $q_{0_{\var x}}$ going through $q$ to $f_{\var x}$ and the smallest $\ell^q_p\leq m$ such that there is a cycle containing $q$ of length $\ell^q_p$. Then, we get that all accepting paths going through $q$ in $B_{\var x}$ as well as in $A_{\var x}$ (and consequently, all the corresponding words) have lengths of the form $\ell^q_p + r \ell^q_c$, for $r\geq 0$. Conversely, there exists an accepting path in $B_{\var x}$ going through $q$ of length $\ell^q_p + r \ell^q_c$ for all $r\geq 0$, and a word $w \in \terminals$ such that $|w| = \ell^q_p + r \ell^q_c$ and  $w \in L(A_{\var x})$. This means that for each variable $\var x$ we get a disjunction of length constraints of the form $\len(\var x) = \ell^q_p+\alpha\ell^q_c$ for some state $q \in Q_{\var x}$ and $r\geq 0$. We add the length restrictions obtained for each variable $\var x$ occurring in each constraint wihtin the list ${\mathcal L}_r$ to the list ${\mathcal L}_n$. 

It remains to check whether the linear arithmetic constraints of ${\mathcal L}_n$ are satisfiable, which is decidable (see~\cite{Chvatal}). If all arithmetic constraints are satisfied, it automatically means that there exists an assignment for each variable $\var x$ such that the regular membership constraints are satisfied too. So, $\varphi$ is satisfiable.

The above non-deterministic algorithm is clearly sound and terminates, but it does not run in polynomial space. There are several steps where we obviously may use exponential space; for instance, computing $\overline{M}_R$ from $M_R$ or computing the intersection automaton $A_{\var x}$. Also, deciding the satisfiability of ${\mathcal L}_n$ can be done in polynomial space w.r.t. $n$ if all the coefficients of the linear constraints in ${\mathcal L}_n$ can be represented in a number of bits polynomial in $n$; thus, we need to show that this holds. 

We will now explain how to implement the algorithm above in polynomial space. 
The first difference occurs when switching from regexes to automata in the list ${\mathcal L}_r$. If $\alpha\rein R$ or $\neg(\alpha\rein R)$ is a regular membership constraint of ${\mathcal L}_{r}$, let $M_R$ be the NFA accepting the language of $R$; the constraint $\alpha\rein R$ is equivalent to $\alpha\in L(M_R)$, while $\neg(\alpha\rein R)$ is equivalent to $\neg(\alpha\in L(M_R))$. So, in this implementation, the list ${\mathcal L}_{r}$ is seen as a list of constraints $\alpha\in L(M)$ or $\neg(\alpha\in L(M))$, where $\alpha$ is a string term and $M$ is an NFA. In this way, we avoid constructing the automaton $\overline{M}_R$ for any regex $R\in\regl{\terminals{}}{}$. We do not need to construct $\overline{M}_R$, as we can simulate it using $M_R$. If we were to construct the DFA $M_R$ (with the powerset construction) and then compute the complement DFA, we would get that the states of $\overline{M}_R$ are tuples of (at most $n$) states of $M_R$, the transitions are the transitions of $M_R$ applied on the components of the tuples, and a tuple is final if and only if none of the states it contains was final in $M_R$. Similarly, instead of working directly with the NFAs $M$ from constraints $\alpha\in L(M)$, we will work with the corresponding DFAs $D(M)$: again, we simulate them, because we know that their states are tuples of states from $M$, and we can simulate their transitions by executing the transitions of $M$ on the components of the tuples.

Following the strategy above, we cannot explicitly construct the NFA $A_{\var x}$ for a variable $\var x$: on the one hand, we have not obtained all the automata we intersected in order to construct $A_{\var x}$, and on the other hand, as a variable might have $O(n)$ occurrences and the NFAs associated with its occurrences have $O(n)$ states, the automaton $A_{\var x}$ may be of exponential size. But even if we cannot effectively build $A_{\var x}$, we can simulate it. In the previous construction, $A_{\var x}$ was obtained as the intersection of the NFAs corresponding to the occurrences of the variable $\var x$; now we will construct $A_{\var x}$ as an intersection of DFAs, so it will also be deterministic. More precisely, its states are tuples of states corresponding to these automata. In such a tuple, the position corresponding to an occurrence of $\var x$ in a constraint $\alpha\in L(M)$ stores a state of $D(M)$, which can be seen as a tuple of states of $M$; the position corresponding to an occurrence of $\var x$ in a constraint $\neg(\alpha\in L(M))$ stores a state of $\overline{M}$, so, once more, a tuple of states of $M$. The transitions in $A_{\var x}$ can be simulated by executing the transitions on components, using the corresponding automata, and a state of $A_{\var x}$ is final if all its components are final. Clearly, the size of each state of $A_{\var x}$ (i.e., the number of elements in each tuple representing a state) is $O(n^2)$. The states, final states, and transitions of $B_{\var x}$ are the same as the ones of $A_{\var x}$. The value of $m$, the number of states of $B_{\var x}$ and $A_{\var x}$, is $O(n^{n^2})$. Note that $m$ can be represented with $O(n^2 \log n)$ bits. Also, note that, in this implementation, $A_{\var x}$ is deterministic, while $B_{\var x}$ is not (all transitions now have the same label). 

Finally, we need to explain how the assignment of each variable $\var x$ is done. We non-deterministically choose a state of $A_{\var x}$ (and $B_{\var x}$) such that the word assigned to $\var x$ labels a path that must go through $q$. Then we non-deterministically guess the corresponding $\ell^q_p$ and $\ell^q_c$ (whose value can be represented on a polynomial number of bits, according to the upper bounds $\ell^q_p \leq 2m$ and $\ell^q_c \leq m$) and check if there is a path of length $\ell^q_p$ from the initial state through $q$ to the final state, and a cycle of length $\ell^q_c$ containing $q$. Finally, we add the constraint $\len(\var x) = \ell^q_p+\alpha\ell^q_c$ for some state $q \in Q_{\var x}$ and $r\geq 0$ to ${\mathcal L}_n$. We do this for all variables.

It can be shown by standard methods (see, e.g., \cite{Chvatal}) that if the integer program defined by ${\mathcal L}_n$ has a solution, then there is a solution contained inside the ball of radius $c^{nL}$ centred in the origin, where $c$ is a constant and $L$ is polynomial in the total number of bits needed to write the coefficients of the constraints in ${\mathcal L}_n$. So, it is enough to look for a solution to ${\mathcal L}_n$ inside the ball of radius $c^{nL}$. This means that the value of each variable from ${\mathcal L}_n$ can be written (in binary) in a polynomial number of bits. It is enough to guess an assignment for these variables (which can be stored in polynomial space) and then check if this is a solution to the integer programming problem (which can be done polynomial space). If the guess was correct, then $\varphi$ is satisfiable. 

This modified implementation now runs in polynomial space, so this concludes the proof for $\folstruc_{slc}$. The result for $\folstruc_{sl}$ follows immediately.

This ends the proof of the upper bound. \qed\end{proof}

\begin{lemma}\label{lem:slc-hardness}
The satisfiability problems for $\folstruc_{slc}$ and $\folstruc_{sl}$ of simple regular expressions, linear integer arithmetic, string length, and concatenation
are $\mathsf{PSPACE}$-hard. 
\end{lemma}
\begin{proof}
The following problem is $\mathsf{PSPACE}$-complete \cite{Kozen77}. 

\begin{quote}
	Let $M_1,\ldots,M_n$ be $n$ NFAs over an alphabet $\terminals{}$ for $n \in \mathds{N}$. 

	Decide whether there exits a word $\alpha \in \terminals{}^*$ such that $\alpha \in \cap_{1\leq i\leq n}L(M_i)$.
\end{quote}
This problem can be reduced to the satisfiability problem for the quantifier-free theory $\folstruc_{sl}$. 

Let $R_i \in\regl{\terminals{}}{}$ be a regex such that $L(R_i)=L(M_i)$, for all $i \in \Set{1,\dots,n}$. The expression $R_i$ can be constructed in polynomial time from $M_i$ (cf. \cite{han2005generalization}). We now define a formula $$\varphi=\bigwedge_{1\leq i\leq n}\,\var x\rein R_i$$ where $\var x \in \variables{}$. Clearly, $\varphi$ is satisfiable if and only if $$\bigcap_{1\leq i\leq n}L(M_i)\neq \emptyset.$$
This ends the proof of the lower bound. 
\qed \end{proof}

\begin{restatable}{theorem}{lsre}
\label{thm:lsre}
The satisfiability problems for $\folstruc_{slc}$ and $\folstruc_{sl}$ of simple regular expressions, linear integer arithmetic, string length, and concatenation
are $\mathsf{PSPACE}$-complete. 
\end{restatable}
\begin{proof}
Immediate consequence of Lemma \ref{lem:slc-pspace} and \ref{lem:slc-hardness}.\qed
\end{proof}
\noindent 
\begin{restatable}{theorem}{lrec}\label{thm:lrec}
The satisfiability problems for $\folstruc_{elc}$ and $\folstruc_{el}$ of regular expressions, linear integer arithmetic, concatenation, and string length are
decidable.
\end{restatable}
\begin{proof}
We can use the same strategy defined above. The only problem is that because the regexes may now contain complements, when constructing the automaton $M_R$ associated with a regex $R \in \regl{\terminals{}}{}$ we might have an exponential blow-up in size, even if the alphabet of the regex resp. NFA is binary and only one complement is used (as shown, for instance, in \cite{HospodarJM18}). We can no longer guarantee the polynomial space complexity of our approach, but the decidability result still holds.
\qed \end{proof}

Let $g$ be defined as given in Remark~\ref{rem:blow_up}. Based on the classical results from \cite{stockmeyer}, we can derive the following theorem.
\begin{restatable}{theorem}{lobound}\label{thm:lobound}
There exists a positive rational number $c$ such that the satisfiability problem for the fragments of $\folstruc_{s}$ and $\folstruc_{sc}$ allowing only formulae of complement-depth at least $k$ is not in $\mathsf{NSPACE}$$(g(k-1,cn))$.
\end{restatable}
\begin{proof}
This is a direct consequence of Remark \ref{rem:blow_up}.
\qed \end{proof}
\begin{restatable}{theorem}{lsren}\label{thm:sren}
The satisfiability problem for $\folstruc_{sn}$ (respectively for $\folstruc_{en}$) of (simple) regular expressions and a string-to-number predicate is $\mathsf{PSPACE}$-complete (respectively decidable).
\end{restatable} 
\begin{proof}
The lower bound follows as in Theorem \ref{thm:lsre}. We now show the $\mathsf{PSPACE}$ upper bound.

Let $\varphi$ be a formula of length $n\in\mathds{N}$ in the theory $\folstruc_{sn}$. 

Firstly, once more, convert $\varphi$ into an equivalent formula $\varphi'$ in negation normal form which consists of a Boolean combination of atoms of the form $\alpha\rein R$ or $\neg(\alpha\rein R)$, where $\alpha\in\pats{\terminals}{}$ and $R \in\regl{\terminals{}}{}$ (thus $R$ is not containing any complement), as well as atoms encoding arithmetic constraints, and $\numstr$ predicates. Clearly, $|\varphi'|\in \mathcal{O}(|\varphi|)$. 

Similar to the proof of Lemma \ref{lem:slc-pspace}, secondly, we non-deterministically guess the truth assignment of all atoms (regular constraints, arithmetic constraints, or $\numstr$ predicates) such that $\varphi'$ evaluates to true. We can construct a list ${\mathcal L}_{r}$ of atoms of the form $\alpha\rein R$ or $\neg(\alpha\in R)$, where $\alpha$ is a string term and $R$ is a simple regular expression, that all have to evaluate to true. We also construct a second list ${\mathcal L}_{n}$ containing a set of arithmetic linear constraints that should all be true. 

Thirdly, we process predicates of the form $\numstr(m,\alpha)$ and $\neg \numstr(m,\alpha)$, where $m$ is an integer term and $\alpha\in\terminals{}^*\cup\variables{}$. Note, since we do not allow concatenation, $\alpha$ can only be a word consisting of constants or a single variable. If $m$ is neither a variable nor a constant, we add a new integer variable $\var x_m$ and replace $\numstr(m,\alpha)$ (respectively, $\neg \numstr(m,\alpha)$) by the predicate $\numstr(\var x_m,\alpha)$ (respectively, $\neg \numstr(\var x_m,\alpha)$) and the arithmetic atom $\var x_m=m$. A similar processing can be done to replace the constant strings from $\numstr$ predicates by variables. We obtain in this way a new formula $\varphi''$, still of size $O(|\varphi|)$. After this, each term in every $\numstr$ predicate is either a constant or variable of the appropriate sort. 

Now, in $\varphi''$, if we have a predicate $\numstr(m,\alpha)$ (respectively, $\neg \numstr(m,\alpha)$) where $m \in \mathds{Z}$ is a constant, we let $M$ be the constant string consisting of the shortest binary representation of $m$. We add $\alpha\rein 0^*M$ (respectively, $\neg(\alpha\rein 0^*M)$) to the list of regular constraints ${\mathcal L}_r$. We remove $\numstr(m,\alpha)$ (respectively, $\neg \numstr(m,\alpha)$) from $\varphi''$. 
If we have $\numstr(\var x,\alpha)$ (respectively, $\neg \numstr(\var x,\alpha)$) where $\var x$ is an integer variable, we add $\alpha\rein 0^*\{0,1\}^*$ (respectively, $\neg(\alpha\rein 0^*\{0,1\}^*)$) to the regular constraints ${\mathcal L}_r$. We remove $\numstr(\var x,\alpha)$ (respectively, $\neg \numstr(\var x,\alpha)$) from $\varphi''$, but store in a new list ${\mathcal L}_{b}$ the information that the binary representation of $\var x$ fulfils the same regular constraints as $\alpha$ (e.g., if we have $\alpha\rein R$ we add $\var x\rein R$ as well), or, respectively, the complement of the regular constraints of $\alpha$. In the latter case, it is worth nothing that if we have a restriction $\neg(\alpha\rein R)$, the binary representation of $\var x$ must be in the language defined by $R$, so we will not obtain regular expressions with stacked complements.

In this way we obtain a list of regular constraints that need to be true, a list of arithmetic linear constraints that need to be true, as well as a list of constraints stating the the binary representation of certain integer variables must also fulfil the same regular constraints as certain variables.

So far, all transformations can be clearly carried out in polynomial space w.r.t. $n$. So, in the $L_n$ list, all coefficients can be represented in a polynomial number of bits, by the same reasons as before. Let $s\in\mathds{Z}$ be the sum of the absolute values of all the constants occurring in the arithmetic constraints. Clearly $s$ can be represented in a number of bits polynomial in $n$. The list ${\mathcal L}_b$ remains unchanged. 

Unlike the algorithm presented in  Lemma~\ref{lem:slc-pspace}, we will not solve the integer linear system defined by ${\mathcal L}_n$ using integer programming tools. Instead, we use the fact that deciding whether the set of linear constraints is satisfiable is equivalent to checking whether the language accepted by a finite synchronized multi-tape automaton $\texttt{A}$ is empty or not (see \cite{GaneshBD02}). This automaton has as states $p$-tuples of integers within the set $\Set{i \in \mathds{Z} | -\ell \leq i \leq \ell}$ for an appropriate $\ell\in\mathds{Z}$, where $p\in\mathds{N}$ is the number of variables occurring in ${\mathcal L}_r$; so, each state can be represented in a polynomial number of bits w.r.t. $n$. Each tape of the automaton corresponds to a variable occurring in the set of linear constraints. For a certain input, the automaton checks whether the binary strings found on the tapes can be used as the binary representations of the corresponding variables in an assignment that satisfies the linear constraints (we assume that the representations of these variables are with leading 0s, so that their least significant bits are aligned). Intuitively, we encode a system of the form $A \vec{\var x} \leq \vec{b}^T$, where $A$ is the coefficient matrix, $\vec{\var x}$ is the (column) vector of all variables, and $\vec{b}$ is the vector of integers. The state of the automaton computes $A \vec{y}(1..i)$, where $\vec{y}$ is a vector of integers whose binary representations are on the tapes of the automata, and $\vec{y}(1..i)$ is the vector containing on each component, respectively, the integer whose representation consists of the most representative $i$ bits from the representation of integer found on the corresponding position in $\vec{y}$. In order for $\vec{y}$ to be a solution, each component of the the computed value $A\vec{y}(1..i)$ must stay between $-\ell$ and $\ell$. Moreover, the transition from a state corresponding to $A\vec{y}(1..i)$ to the state corresponding to $A\vec{y}(1..i+1)$ can be computed from the current state and $\vec{y}(i+1)$. More details on the construction of $A$ are given in \cite{GaneshBD02}.

In our case, we also need to enforce that both the regular constraints on the binary representation of certain variables from ${\mathcal L}_b$, as well as the the regular constraints on the other string variables that are not involved in any $\numstr$ predicate, are fulfilled. 

Therefore, we augment the automaton described above in the following way:
\begin{enumerate*}
\item We will add a tape for each string variable $\var y \in \variables{}$ which does not occur in any $\numstr$ predicate, but appears in a linear length or regular membership constraint. From the point of view of the arithmetic part implemented by the automaton, these tapes are treated as if they represent variables which occur with coefficient $0$ in the equations of the linear system we want to solve. In this context, as the letters on the respective tape are not involved in any arithmetic operation, we do not need to restrict the respective letters to the bits $\Set{0,1}$. 
\item We assume that the words on the tapes of the automaton have their last letters aligned. To this end, we can assume that our strings are padded with a prefix of special blank symbols, so that they all have the same length. The arithmetic part implemented by the automaton treat these blanks as $0$s. The part of the automaton which checks the regular constraints simply neglects these blanks.
\end{enumerate*}

We will now explain how the automaton works. The arithmetic part was already described. The part checking the regular constraints works as follows.

Suppose that the binary representation of the variable $\var x$, which corresponds to the $j^{th}$ tape of $\texttt{A}$, must be in the language defined by a regex $R\in\regl{\terminals{}}{}$ (or, alternatively, not in the language defined by $R$). To this end, while $\texttt{A}$ reads the representation of $\var x$, as soon as we reach the first non-blank symbol on that tape, we also simulate the computation on $\var x$ of the deterministic automaton corresponding to $R$ or, respectively, to $ \overline{R}$ (i.e., we construct the NFA $M_R$ for $R$, and then simulate the transitions corresponding to $R$ or $\overline{R}$ on the DFA obtained via the powerset construction from $M_R$ as in the proof of Theorem \ref{thm:lsre}). We accept the representations given as input to $\texttt{A}$ if and only if $\texttt{A}$ accepts them \textit{and} they are also accepted by the automaton checking the components corresponding to variables that occur in ${\mathcal L}_b$. We simulate the states of the automata to enforce the constraints from ${\mathcal L}_b$ have a polynomial number of components, and their total number is also polynomial. 

The tapes corresponding to string variables which do not occur in $\numstr$ predicates are processed similarly. We simply simulate the computation of the deterministic automaton corresponding to the regular constraint on the string found on that tape.

To check whether there exists an input accepted by $\texttt{A}$ in this way, we non-deterministically guess an input for $\texttt{A}$, by selecting one by one, from the most representative (left) to the least representative (right), the letters on the tapes of the automaton (and storing at each step just the current guess, without saving the past guesses), and keep track of the current state of all the automata we simulate. This process is clearly correct from the information we gave, and all the information we store fit in a polynomial number of bits. However, it is not clear that it terminates. 

For this, we show that we can bound the number of states of the automaton by a polynomial. 
Each of the automata corresponding to regular constraints, which we run when using $\texttt{A}$, has $\leq P(n)^{Q(n)}$ states for some polynomials $P$ and $Q$, and we run them in parallel on the tapes of $\texttt{A}$. The state corresponding to the linear system is a number between $-\ell$ and $\ell$. We accept a guessed content of the tapes if and only if all automata accept it and the linear system is satisfied. Hence, we are essentially simulating a run of the product of these $\leq Z(n)$ automata, where $Z$ is another polynomial. This product automaton accepts a non-empty language if and only if it accepts a word (sequence of columns of bits) whose length is at most its number of states. So, we must check whether it accepts a sequence of columns of bits of length $P(n)^{Q(n)Z(n)}$. By keeping a binary counter, we can count how many guesses we have made, and stop (without having found a word) when we need to use more than $Q(n)Z(n) \log P(n)$ bits for this counter.

If we can guess an assignment of the variables that satisfies ${\mathcal L}_b$ and ${\mathcal L}_n$ and the remaining regular constraints then $\varphi$ is satisfiable. If we cannot find any assignment, $\varphi$ is not satisfiable.

Clearly, the decision procedure described works also in the case when the regular expressions contain complements. However, we cannot show the polynomial upper bound on the space we use. Therefore, $\folstruc_{en}$ is decidable. \qed
\end{proof}
\begin{restatable}{theorem}{decfour} \label{thm:undecWE}
The satisfiability problem for $\folstruc_{slnc}$ of regular expressions, linear integer arithmetic, a string-to-number predicate and concatenation is undecidable.
\end{restatable}
\begin{proof}
We begin by looking at the theory $\folstruc_{snc}$ and define a predicate $eqLen \subseteq \pats{\terminals{}}{} \times \pats{\terminals{}}{}$ defined by $$eqLen(\alpha,\beta) \text{ iff } \len(\alpha)=\len(\beta)$$ for $\alpha,\beta\in\pats{\terminals{}}{}$.
We can express $eqLen(\alpha,\beta)$ as:
 \begin{align*}
 	eqLen(\alpha,\beta)\; =&\,\hspace*{0.3cm}(\var z\in 1\{0\}^*)\\ &\land \numstr(\var i, \var z) \land \numstr (\var j,\var z0) \land \numstr(\var n_a,1\alpha)
 						  \land \numstr(\var n_b,1\beta) \\
 &\land (\var i\leq n_a) \land (\var n_a +1 \leq \var j) \land (\var i\leq \var n_b) \land (\var n_b +1 \leq \var j),
 \end{align*}
for integer variables $\var i, \var j, \var n_a, \var n_b$ and string variables $\var z$.
Indeed, for a potential assignment $\substitution{}\in\assignments{\terminals{}\cup\mathds{Z}}$, we have $$\substitution(\var i)=2^{len(\var z)} \text { and  }\substitution(\var j)=2^{\len(\var z)+1}.$$ Then, we have $$\substitution(\var n_a)=2^{\len(\alpha)}+A \text{ and }\substitution(\var n_b)=2^{\len(\beta)}+B,$$ where $\numstr(A,\alpha)$ and $\numstr(B,\beta)$ are true. Therefore, $$2^{\len(\var z)}\leq 2^{\len(\alpha)}+A<2^{\len(\var z+1)} \text{ and } 2^{\len(\var z)}\leq 2^{\len(\beta)}+B<2^{\len(\var z)+1}.$$ It is immediate that $\len(\alpha)=\len(\beta)=\len(z)$, so our claim holds. 

We can also show that the theory of word equations with regular constraints and $\numstr$ predicate is equivalent to the theory $\folstruc_{enc}$. 

\medskip

For one direction, we need to be able to express an equality predicate between string terms $eq \subseteq \pats{\terminals{}}{} \times \pats{\terminals{}}{}$. The regular constraints as well as those involving the $\numstr$ predicate are canonically encoded. 

This predicate is encoded as follows:

\begin{align*}
	eq(\alpha,\beta) &= eqLen(\alpha,\beta) \land \numstr(\var i,1\alpha1\beta) \land \numstr(\var j,1\beta1\alpha) \land (\var i = \var j),
\end{align*}
 for $\alpha,\beta \in \pats{\terminals{}}{}$.
Indeed, this tests for a potential assignment $\substitution{}\in\assignments{\terminals{}\cup\mathds{Z}}$ that 
$$\len(\alpha)=\len(\beta) \text{ and } \substitution{}(1\alpha1\beta)=\substitution(1\beta1\alpha).$$ If these are true, it is immediate that $\substitution(\alpha)=\substitution(\beta)$. 

\medskip

For the converse, it is easy to see that each string constraint $\alpha\rein R$ (respectively, $\lnot(\alpha\rein R)$), where $\alpha\in\pats{\terminals}{}$and $R\in\regcl{\terminals{}}{}$, can be expressed as the word equation $\alpha \doteq \var x_R$, where $\var x_R \in \variables{}$ is a fresh variable, which is constrained by the regular language defined by $R$ (respectively, by the regular language defined by $\overline{R}$). 

This allows us to define a stronger length-comparison predicate.
We will define a predicate $leqLen\subseteq \pats{\terminals{}}{} \times \pats{\terminals{}}{}$, whose semantics is defined by $$leqLen(\alpha,\beta)\text{ iff } \len(\alpha)\leq\len(\beta),$$ for $\alpha,\beta \in \pats{\terminals{}}{}$. We can express $leqLen(\alpha,\beta)$ by
$$leqLen(\alpha,\beta) = (\var z\in \{0,1\}^*) \land eqLen(\alpha \var z, \beta).$$



Finally, we can now move on to $\folstruc_{elnc}$ and show our statement. According to \cite{rp2018strings} the quantifier-free theory of word equations expanded with $\numstr$ predicate and length function (not only a length-comparison predicate) and linear arithmetic is undecidable. Thus, if we consider $\folstruc_{elnc}$, this undecidability result immediately holds according to the above.

\qed \end{proof}
\bibliographystyle{splncs04}
\begin{btSect}{words_app}
\section*{Additional References used in the Appendix}
\btPrintAll
\end{btSect}

\end{document}